\documentclass[11pt]{article}

\usepackage{authblk}
\usepackage{booktabs,longtable,xcolor}
\definecolor{darkblue}{RGB}{0,70,140}

\usepackage[T1]{fontenc}
\usepackage[utf8]{inputenc}
\usepackage{lmodern}
\usepackage{microtype}
\usepackage[margin=1in]{geometry}
\usepackage{authblk}
\newcommand{\jm}{j_{\min}}
\newcommand{\jM}{j_{\max}}

\newcommand{\equalcontrib}{\textsuperscript{\(\dagger\)}}
\newcommand{\corresponding}{\textsuperscript{*}}

\usepackage{amsmath}
\usepackage{amssymb}
\usepackage{amsthm}
\usepackage{mathtools}

\usepackage{graphicx}
\usepackage{booktabs}
\usepackage{multirow}
\usepackage{threeparttable}
\usepackage{caption}
\usepackage{subcaption}

\usepackage[table]{xcolor}
\usepackage[most]{tcolorbox}

\definecolor{darkblue}{RGB}{0,70,140}
\definecolor{darkred}{RGB}{150,0,0}
\definecolor{darkgray}{RGB}{90,90,90}

\definecolor{theorembackground}{RGB}{239,248,255}
\definecolor{theoremborder}{RGB}{30,110,180}
\newtcolorbox{statementbox}{
    enhanced,
    breakable,
    colback=theorembackground,
    colframe=theoremborder,
    boxrule=0.7pt,
    arc=2pt,
    left=6pt,
    right=6pt,
    top=5pt,
    bottom=5pt,
    before skip=8pt,
    after skip=8pt
}

\definecolor{proofbackground}{RGB}{255,250,240}
\definecolor{proofborder}{RGB}{170,145,105}

\newtcolorbox{proofbox}{
    enhanced,
    breakable,
    colback=proofbackground,
    colframe=proofborder,
    boxrule=0.6pt,
    arc=2pt,
    left=6pt,
    right=6pt,
    top=5pt,
    bottom=5pt,
    before skip=8pt,
    after skip=8pt
}

\usepackage[most]{tcolorbox}

\definecolor{modelbackground}{RGB}{242,250,248}
\definecolor{modelborder}{RGB}{35,125,115}

\newtcolorbox{modelbox}[1]{
    enhanced,
    breakable,
    colback=modelbackground,
    colframe=modelborder,
    colbacktitle=modelborder,
    coltitle=white,
    fonttitle=\bfseries,
    title={#1},
    boxrule=0.7pt,
    arc=2pt,
    left=7pt,
    right=7pt,
    top=6pt,
    bottom=6pt,
    before skip=10pt,
    after skip=10pt
}

\theoremstyle{plain}
\newtheorem{theorem}{Theorem}[section]
\newtheorem{proposition}[theorem]{Proposition}
\newtheorem{lemma}[theorem]{Lemma}
\newtheorem{corollary}[theorem]{Corollary}

\theoremstyle{definition}

\theoremstyle{remark}

\newenvironment{boxedtheorem}
{\begin{statementbox}\begin{theorem}}
{\end{theorem}\end{statementbox}}

\newenvironment{boxedcorollary}
{\begin{statementbox}\begin{corollary}}
{\end{corollary}\end{statementbox}}

\newenvironment{boxedproposition}
{\begin{statementbox}\begin{proposition}}
{\end{proposition}\end{statementbox}}

\usepackage[round,authoryear]{natbib}

\usepackage[
    colorlinks=true,
    linkcolor=darkblue,
    citecolor=darkred,
    urlcolor=darkblue
]{hyperref}

\usepackage[
    capitalise,
    nameinlink,
    noabbrev
]{cleveref}

\usepackage{tikz}
\usetikzlibrary{
    arrows.meta,
    positioning,
    fit,
    backgrounds,
    calc,
    shapes.geometric
}

\definecolor{annotatorblue}{RGB}{52,101,164}
\definecolor{itemorange}{RGB}{230,126,34}
\definecolor{branchgreen}{RGB}{39,139,94}
\definecolor{posteriorpurple}{RGB}{106,61,154}
\definecolor{lightgray}{RGB}{245,245,245}
\definecolor{darkgray}{RGB}{75,75,75}
\usepackage{xspace}
\usepackage{pifont}

\newcommand{\best}[1]{\ensuremath{\boldsymbol{#1}}}
\newcommand{\second}[1]{\ensuremath{\underline{#1}}}

\newcommand{\indic}[1]{\mathbf 1\!\left\{#1\right\}}

\DeclareRobustCommand{\Method}{%
    \ifmmode
        \text{\normalfont CC-Rasch}%
    \else
        CC-Rasch%
    \fi
    \xspace
}

\title{
A Model for Imbalanced Label Aggregation:
A Focus on Minority-Class Detection
}

\author[1,2]{
Gabriel Singer\equalcontrib\corresponding
}

\author[2]{
Samuel Gruffaz\equalcontrib\corresponding
}

\author[1]{
Olivier Vo Van
}

\author[2]{
Nicolas Vayatis
}

\author[2]{
Argyris Kalogeratos
}

\affil[1]{
SNCF, Paris, France
}
\affil[2]{Tampere University, Finland, Biostatistics group}

\affil[3]{
Université Paris-Saclay,
ENS Paris-Saclay,
Centre Borelli, France
}

\date{}

\begin{document}

\maketitle

\begin{center}
\small
\(*\) Equal contribution.
\qquad
\(\dagger\) Corresponding authors.\\[2pt]
\href{mailto:gabriel.singer@sncf.fr}{gabriel.singer@sncf.fr}
\qquad
\href{mailto:samuel.gruffaz@tuni.fi}
{samuel.gruffaz@tuni.fi}
\end{center}
\maketitle

\begin{abstract}
We study imbalanced crowdsourcing with a focus on class-dependent annotator accuracy, a setting that, to the best of our knowledge, remains relatively underexplored despite its importance in real-world inspection systems where the labels of greatest operational importance are also the rarest ones.
In this setting, annotators may be reliable on both classes, unreliable on both classes, majority-class specialists, or minority-class specialists. Existing models only partially address this problem: they either capture class-dependent errors but ignore item difficulty, or they model item difficulty without capturing class-dependent errors. To fill this gap for imbalanced datasets in crowdsourcing, we introduce a generative aggregation model combining item difficulty with class-dependent annotator competence. The model allows both annotator abilities and item difficulties to vary across classes.  
We then revisit Condorcet's Jury Theorem in the class-imbalanced setting. We also show that majority voting asymptotically preserves the underlying class proportion. 
We evaluate our model on $33$ real-world crowdsourcing datasets, covering multiclass tasks such as images and text, as well as two large-scale regimes: large-scale annotation datasets, with many annotations per item, and
large-scale item datasets, with a large number of annotated instances. Across these diverse settings, our model consistently achieves the highest minority recall while remaining competitive in balanced accuracy, making it particularly relevant when rare-label recovery is the primary objective.
\end{abstract}

\section{Introduction}
\label{sec:introduction}

In many real-world inspection systems, the labels of  interest are also the rarest ones. 
This creates a tension: large annotated datasets are needed to train reliable models, but expert annotation of rare events is costly, and difficult to scale. 
Crowdsourcing provides an attractive alternative by collecting multiple labels from non-expert annotators at lower cost and larger scale \citep{snow2008cheap}.  In a balanced setting, majority voting and weighted majority voting enjoy consistency guarantees \citep{berend2014consistencyweightedmajorityvotes,BandParoushMajVoteCharach}. 
However, imbalanced  settings break part of this picture. Majority voting can achieve high overall accuracy, while  missing rare but critical minority-class defects.

Class imbalance also changes the structure of the crowd. In balanced binary settings, it is  natural to distinguish reliable from unreliable annotators \citep{berend2014consistencyweightedmajorityvotes, BandParoushMajVoteCharach}. In imbalanced settings, this distinction is not enough to capture all possibilities. 
In fact, an annotator may be reliable on both classes, unreliable on both classes, reliable only on the majority class, or only on the minority class.
Dawid--Skene's model captures annotator confusion matrices and can therefore represent class-dependent error patterns \citep{dawid1979maximum}. Yet, this does not explicitly model item difficulty. 
The Rasch model was the first to introduce item difficulty and annotator ability modeling \citep{rasch1960probabilistic}; GLAD later adapted this principle by adding the estimation of the latent true labels \citep{whitehill2009glad}.
However, they both consider a single ability parameter per annotator, which prevents them from distinguishing majority-class competence from minority-class competence. Thus, Dawid--Skene is class-dependent, but not difficulty-aware, while GLAD is difficulty-aware but not class-dependent. 
Moreover, most aggregation methods \cite{Ustalov_2024,crowdfm} are accuracy-driven rather than designed for minority-class recovery.

\paragraph{Contributions.} Our contributions are as follows:
\begin{enumerate}
\item \textbf{Theoretical analysis of majority vote under class imbalance.}
We provide interpretable conditions guaranteeing a Condorcet jury theorem in
the class-imbalanced setting (Theorem~\ref{Condorcet}) and characterize
how majority voting asymptotically preserves the underlying minority-class
proportion
(Corollary~\ref{mvimbalancedratio}).

\item \textbf{Mixed-effects generative model.}
We propose a new generative model of annotator competence and item difficulty that captures how both vary across classes. We complete our analysis by proving that, the M-step of the EM  algorithm admits a unique solution (Proposition~\ref{prop:unique-m-step}).

\item \textbf{Large-scale empirical evaluation.}
We evaluate \Method against four baselines on \(33\) diverse
real-world crowdsourcing datasets, constituting, to the best of our knowledge,
one of the largest benchmarks.
\Method achieves consistent gains in minority recall across evaluation regimes,
while controlled synthetic experiments confirm its robustness as
minority-class specialists become increasingly rare.
\end{enumerate}

In what follows: Section~\ref{sec:background} reviews the Rasch and GLAD models. Section~\ref{theoreticalMV} analyzes majority voting under class imbalance. Section~\ref{sec:theModel} introduces \Method, and Section~\ref{sec:exps} evaluates it on synthetic and real-world datasets. Section~\ref{sec:conclusions} concludes the paper. Complete proofs are provided in the supplementary material.

\section{Background}
\label{sec:background}
\paragraph{The crowdsourcing problem}
We consider that each item has an unknown ground-truth label \(Y\in[K]\), where
\([K]:=\{0,\ldots,K-1\}\) and \(K\) denotes the number of classes. As a reference example, one may think an item to be an image and of its ground-truth label as indicating whether an object appears in it; in this case $K=2$. Let \(\mathcal R_R=\{1,\dots,R\}\) be the set of \(R\) annotators. Each item receives at most \(R\) noisy labels \(Y_1,\ldots,Y_R\in[K]\).
A label aggregation rule combines these noisy labels into a single prediction
\(\widehat Y\in[K]\). The simplest aggregation rule is majority
voting, which outputs the class with most votes:
\[
\widehat Y_R^{\mathrm{MV}}
\in
\operatorname*{arg\,max}_{k\in[K]}
\sum_{r=1}^{R}\indic{Y_r=k}.
\]
More generally, probabilistic aggregation methods such as GLAD or Rasch that will be presented later, estimate the posterior
distribution \(\mathbb P(Y=k\mid Y_1,\ldots,Y_R)\) and predict the most likely
class; \[
\widehat Y
\in
\operatorname*{arg\,max}_{k\in[K]}
\mathbb P\!\left(Y=k\mid Y_1,\ldots,Y_R\right).
\]

\paragraph{Rasch and GLAD models.}
This work belongs to the family of probabilistic approaches that jointly model annotator competence and item difficulty. One of the earliest and most influential models of this type is the Rasch model \citep{rasch1960probabilistic}, which is widely used in item response theory.
Consider a set of $m$ annotators and $n$ items, with $[n]:=\{1,\ldots,n\}$. Each annotator $r\in[m]$ is associated with an ability parameter $\theta_r^\star\in\mathbb{R}$, while each item $i\in[n]$ is associated with a difficulty parameter $\beta_i^\star\in\mathbb{R}$. 
Let \(Y_{ri}\) denote the response provided by respondent \(r\) to item \(i\), and let \(Y_i\) denote the corresponding correct answer. We define the correctness indicator as
\begin{equation}
C_{ri}
:=
\indic{Y_{ri}=Y_i}
\in\{0,1\},
\end{equation}
where, for an event \(A\), \(\indic{A}\) equals \(1\) if \(A\) occurs and \(0\) otherwise.
 The Rasch model assumes: 
\begin{equation}
C_{ri}
\sim
\operatorname{Bernoulli}
\left(
\sigma(\theta_r^\star-\beta_i^\star)
\right),
\label{Rasch-equation}
\end{equation}
where for any $x\in \mathbb{R},$
\begin{equation}\label{sigmoid}
\sigma(x)
:=
\frac{1}{1+\exp(-x)}, 
\end{equation}
is the logistic function. 
Thus, the probability of a correct response increases with the difference between the annotator's ability and the item's difficulty. In particular, when $\theta_r^\star-\beta_i^\star$ is large, this probability is close to one.

In educational testing, for example, $C_{ri}=1$ indicates that student $r$ correctly answers question $i$, $\theta_r^\star$ represents the student's ability, and $\beta_i^\star$ represents the difficulty of the question. A limitation of the Rasch model in crowdsourcing applications is that the correctness indicators $C_{ri}$ can only be observed when the correct answers, or equivalently the true item labels, are known. In standard label aggregation problems, however, these true labels are precisely the quantities that must be inferred.

The GLAD model \citep{whitehill2009glad} extends the ability--difficulty principle to crowdsourcing by treating the true item labels as latent variables and estimating them jointly with annotator abilities and item difficulties. Let $
\mathcal{O}
\subseteq
[m]\times[n]$
denote the set of observed annotator--item pairs. For each item $i\in[n]$, define the set of annotators who provided a label for item $i$:
\begin{equation}
\Omega_i
:=
\left\{
r\in[m]:(r,i)\in\mathcal{O}
\right\}.
\end{equation}

Each item $i$ has an unknown binary true label $Y_i\in \{0,1\}$. For every observed pair $(r,i)\in\mathcal{O}$, annotator $r$ provides a noisy label $Y_{ri}\in\{0,1\}$.
Each annotator $r$ is associated with an ability parameter $\alpha_r\in\mathbb{R}$:
when $\alpha_r>0$, the annotator is more likely to provide the correct label than an incorrect one; when $\alpha_r=0$, the annotator behaves as a random guesser; and when $\alpha_r<0$, the annotator is more likely to provide an incorrect label. Moreover, each item $i$ is associated with an inverse-difficulty parameter $\beta_i>0$, so larger values correspond to easier items.

Conditionally on the latent true labels and the model parameters, the observed annotations are assumed to be independent. GLAD models the probability that annotator $r$ correctly labels item $i$ in a very similar way as Rasch model does.
Given $
C_{ri}
:=
\indic{Y_{ri}=Y_i},$
they define
\begin{equation}
C_{ri}
\mid
Y_i\sim
\operatorname{Bernoulli}
\left(
\sigma(\alpha_r\beta_i)
\right).
\label{eq:glad-correct-prob1}
\end{equation}

Note that easier items having larger $\beta_i$ values amplify the effect of annotator expertise. 
Unlike in the Rasch model, the correctness indicators $C_{ri}$ are not directly observed in GLAD because the true labels $Y_i$ are unknown. The latent labels, annotator abilities, and item inverse difficulties must therefore be estimated jointly from the observed annotations.
\paragraph{Other aggregation models.}
Dawid--Skene \citep{dawid1979maximum} models each annotator through a
class-dependent confusion matrix and jointly estimates the latent labels and
annotator reliabilities using the EM algorithm. Unlike GLAD, it captures
class-dependent annotation behavior but does not model item difficulty.
More recently, CrowdFM \citep{crowdfm} introduced a foundation-model-based
approach to label aggregation. 
\paragraph{Imbalanced crowdsourcing.}
Very few works have explicitly investigated class imbalance in crowdsourcing, which is the setting this work focuses on. \citet{zhang} introduced the imbalanced multiple noisy labeling setting, in which annotators may exhibit different class-conditional labeling accuracies, and proposed the PLAT algorithm.  More recently, \citet{gilgonzalez2026focal} proposed the CCGPFL algorithm, a correlated Gaussian-process model trained with a weighted variational objective. Unlike classical label aggregation methods, CCGPFL additionally relies on item features to model instance-dependent annotator reliability.

\section{Majority voting under class-imbalance
}\label{theoreticalMV}
 Let
	the class-imbalance proportion be defined as:
\[
    \mu:=\min_{k\in [K]}\mathbb{P}(Y=k)
		\,\in\ \,]0,\frac{1}{K}].
\]

We denote by $\jm$ and $\jM$, respectively, the minority
and majority classes, i.e. \(
\jm := \operatorname*{arg\,min}_{k\in[K]}
\mathbb{P}(Y=k)\) and 
\(\jM := \operatorname*{arg\,max}_{k\in[K]}
\mathbb{P}(Y=k)\).
\paragraph{Crowd decomposition.}
For each annotator \(r\in\mathcal R_R\), we define the probability of correctly
labeling an item of class \(k\) as
\[
    p_r^{(k)}
    :=
    \mathbb P(Y_r=k\mid Y=k).
\]
and, we define for each class $k\in [K]$ the subsets of good and bad annotators:
\[
\begin{aligned}
\mathcal G_R^{(k)}
&:=\{r\in\mathcal R_R:\ p_r^{(k)}>\tfrac12+\delta_{k}\},\\
\mathcal B^{(k)}_R
&:=\{r\in\mathcal R_R:\ U^{(k)}\le p_r^{(k)}\le \tfrac12\}.
\end{aligned}
\] Here, \(\delta_k\in [0,\frac{1}{2}]\) is a minimum competence above random
guessing for class \(k\), while \(U^{(k)}\in [0,\frac{1}{2}[\) is a uniform (with respect to $R$) lower bound on
the class-\(k\) accuracies of the annotators.
Moreover, we call an annotator \(r\in\mathcal R_R\) a \emph{minority-class specialist} when
\begin{equation}\label{def:minority-specialist}
r\in\mathcal G_R^{(\jm)},
\end{equation}
and respectively a majority-class specialist when $r\in\mathcal G_R^{(\jM)}$.

For example, consider the binary case $K=2$ and assume that the minority class is indexed by $\jm=0$. Then, $\mathcal{G}_{R}^{(0)}$ denotes the minority specialists, while $\mathcal{G}_{R}^{(1)}$ denotes the majority specialists.

\paragraph{Imbalanced Condorcet Jury Theorem.}
We first establish conditions on the crowd composition under which majority
voting asymptotically recovers every class, including the minority classes,
regardless of the severity of the underlying class imbalance.

For each \(k\in[K]\), let
\(
e_R(k)
:=
\mathbb P\left(
\widehat Y_R^{\mathrm{MV}}\neq Y
\mid Y=k
\right)\) the conditional error and 
\(\pi_k:=\mathbb P(Y=k)
\) the ground truth distribution. 

Throughout the paper, we assume that the random variables
\(\left(Y_r\right)_{r\in\mathcal R_R}\) are conditionally independent given
\(Y\).

\begin{boxedtheorem}\label{Condorcet}
    Assume that for any $k\in [K]$
    \begin{equation}\label{conditionclasseMV}
      \liminf_{R\to \infty}\frac{ \big| \mathcal{G}^{(k)}_{R}\big|}{R} (1+2\delta_{k}) + 2\frac{\big|\mathcal{B}^{(k)}_{R}\big |}{R}U^{(k)}>1.
    \end{equation}
		It follows that
$$\lim_{R\to  \infty}\mathbb{P}(\widehat Y_R^{\textup{MV}}\neq Y)=0.$$
\end{boxedtheorem}
\begin{proofbox}
For every \(R\geq 1\), the law of total probability gives
\[
\mathbb P\!\left(\widehat Y_R^{\mathrm{MV}}\neq Y\right)
=
\sum_{k\in[K]}
\pi_k e_R(k).
\]
Under condition~\eqref{conditionclasseMV}, Proposition~D.3 of
\citet{singer2026optimalfairaggregationcrowdsourced} applies to each class
\(k\in[K]\), yielding
\[
\lim_{R\to\infty}e_R(k)=0.
\]
Since \(K\) is finite,
\[
\lim_{R\to\infty}
\mathbb P\!\left(\widehat Y_R^{\mathrm{MV}}\neq Y\right)
=
\sum_{k\in[K]}
\pi_k
\lim_{R\to\infty}e_R(k)
=0.
\]
Since for any $R\geq 1$ $$\mathbb P\!\left(\widehat Y_R^{\mathrm{MV}}=Y\right)=1-\mathbb P\!\left(\widehat Y_R^{\mathrm{MV}}\neq Y\right),$$ it follows that:
\[
\lim_{R\to\infty}
\mathbb P\!\left(\widehat Y_R^{\mathrm{MV}}=Y\right)
=1.
\]
\end{proofbox}

This is a generalization to the multidimension and multiclass case of Proposition D.$3$ in \citet{singer2026optimalfairaggregationcrowdsourced}.

\textbf{Interpretation:}
Condition~\eqref{conditionclasseMV} provides
for each class \(k\), a lower bound on the number of class-\(k\) specialists
required to compensate for the presence of adversarial annotators.
Indeed, for sufficiently large \(R\), the condition implies
\[
\big|\mathcal G_R^{(k)}\big|
>
\frac{
R-2\big|\mathcal B_R^{(k)}\big|U^{(k)}
}{
1+2\delta_k
}.
\]
In the extreme case \(U^{(k)}\approx 0\), the contribution of the
non-specialists becomes negligible, and the condition approximately reduces to
\[
\big|\mathcal G_R^{(k)}\big|
>
\frac{R}{1+2\delta_k}.
\]
Therefore, if the crowd is made of the non-specialists that are arbitrarily inaccurate on class \(k\),
consistency requires to have a sufficiently large fraction of specialists.
\paragraph{Asymptotic preservation of class imbalance.}
A desirable property of an aggregator is to preserve the underlying class
distribution. This type of property has previously been studied by \citet{zhang} in the binary
case with a homogeneous crowd. Our result extends their analysis to
heterogeneous multiclass crowds.

    \begin{boxedcorollary}\label{mvimbalancedratio}
    Let $R\geq 1$ and $\mu$ be the imbalanced proportion of the ground truth $Y$. 
    Define the majority vote imbalanced proportion: for any $R\geq 1,$ $$\mu^{\textup{MV}}_{R}:=\min_{k\in [K]}\mathbb{P}\left(\widehat Y^{\textup{MV}}_{R}=k\right).$$
    Assume that Eq.~\eqref{conditionclasseMV} holds, then:
    $$\lim_{R\to \infty}\mu^{\textup{MV}}_{R}=\mu.$$
    \end{boxedcorollary}

\begin{proofbox}
Since \eqref{conditionclasseMV} holds: $\lim_{R\to \infty}\mathbb{P}\left(\widehat Y^{MV}_{R}\neq Y\right)=0$. Plus, for any $R\geq 1$ and any $k\in [K]$ we have that 
$$\left|\mathbb{P}\left(\widehat Y^{MV}_{R}=k\right)-\mathbb{P}\left(Y=k\right)\right|\leq \mathbb{P}\left(\widehat Y^{MV}_{R}\neq Y\right),$$ thus for any $R\geq 1:$
$$m_{R}:=\max_{k\in K}\left\{\left|\mathbb{P}\left(\widehat Y^{MV}_{R}=k\right) -\mathbb{P}\left(Y=k\right)\right|\right\}\underset{R\to \infty}{\rightarrow} 0.$$
We conclude by noticing that, for any $R\geq 1$:
\[
\begin{aligned}
\left|\mu_R^{\mathrm{MV}}-\mu\right|
&=
\left|
\min_{k\in[K]}
\mathbb P\!\left(\widehat Y_R^{\mathrm{MV}}=k\right)
-
\min_{k\in[K]}
\mathbb P(Y=k)
\right| \\
&\leq
\max_{k\in[K]}
\left|
\mathbb P\!\left(\widehat Y_R^{\mathrm{MV}}=k\right)
-
\mathbb P(Y=k)
\right|\underset{R\to \infty}{\rightarrow} 0.
\end{aligned}
\]
\end{proofbox}

\section{The \Method Model}\label{sec:theModel}

\subsection{Motivation}
Theorem~\ref{Condorcet} essentially shows that majority voting
requires the crowd to be sufficiently competent on each class. In particular,
performing sufficiently well on the minority class may be difficult to guarantee in practice under class imbalance. We therefore propose a probabilistic model that accounts for both annotation difficulty and the
class-dependent ability of each annotator to recover the ground truth. 

Our second motivation comes from the observation that GLAD does not distinguish
the probability of error of an annotator depending on the value taken by the latent ground truth label. Consequently, GLAD cannot account for
class-dependent annotation behaviors, which are particularly important in the
imbalanced setting. More formally we have:

    \begin{boxedproposition}
   Under the GLAD model Eq.~\ref{eq:glad-correct-prob1}, the sensitivity is equal to specificity. For any $r\in [n]$ and for any $i\in \Omega_r$, we have that:
    $$ \mathbb{P}\left(Y_{ri}=1\mid Y_i=1\right)=\mathbb{P}\left(Y_{ri}=0\mid Y_i=0\right)$$
\end{boxedproposition}

\subsection{The model}
\label{subsec:model_presentation}
We introduce \Method, which stands for \emph{Class-Conditional Rasch}. 
Let $\sigma$ be the sigmoid function (Eq.~\eqref{sigmoid}) and $Y_i\in\{0,\dots,K-1\}$ the latent true class of item $i$, $Y_{ir}\in\{0,\dots,K-1\}$ the label given by annotator $r$ on item $i$. Let, for all $(i,r)\in \Omega_{r}\times[R]$ \(Z_{ir}:=\mathbf 1\{Y_{ir}=Y_i\}\) the indicator that annotator $r$ recovers the true class of item $i$. We assume that:
\[
Z_{ir}\mid Y_i=k
\sim
\mathrm{Bernoulli}(p_{ir,k}),
\quad
p_{ir,k}=\sigma(\alpha_{r,k}-\beta_{i,k}).
\]
 Inspired by mixed-effects models \citep{mun2021mixed}, we split each term into a class-level average and an individual deviation from it:
\[
\alpha_{r,k}=\mu_{\alpha,k}+g_{r,k},
\qquad
\beta_{i,k}=\mu_{\beta,k}+h_{i,k},
\]
where $\mu_{\alpha,k}$ is the crowd average ability. While  $\mu_{\beta,k}$ is the average item difficulty for class $k$. $g_{r,k}$ and $h_{i,k}$ capture how far annotator $r$ and item $i$ deviate from their respective average.
To 
interpret $\mu_{\alpha}$ and $\mu_{\beta}$ as the average ability (respectively difficulty) we impose for any $k \in [K]$: $\sum_{i}h_{i,k}=\sum_{r}g_{r,k}=0$.

Following \citet{whitehill2009glad, crowdfm}, we fix this translation invariance with Gaussian priors:
\[
g_{r,k} \sim \mathcal{N}(0,\sigma_{\alpha,k}^2),
\qquad
h_{i,k} \sim \mathcal{N}(0,\sigma_{\beta,k}^2).
\]
The conditional law of the observed annotation follows:
\begin{equation}\label{PhiRashK}
    \mathbb P(Y_{ir}=y\mid Y_i)
=
p_{ir,k}^{\mathbf 1\{y=Y_i\}}
\left(
\frac{1-p_{ir,k}}{K-1}
\right)^{\mathbf 1\{y\neq Y_i\}},
\end{equation}
i.e.\  we spread the errors uniformly over the $K-1$ remaining classes.

We estimate the model parameters using the Expectation-Maximization (EM) algorithm \citep{dempster1977em}, a widely used approach for maximum-likelihood estimation in latent-variable models. 
EM alternates between computing the conditional distribution of the latent variables given the observed data and the current parameter estimates (E-step), and maximizing an auxiliary function 
that corresponds to the conditional expected log-likelihood (M-step).
 We now detail these two steps for our model.

The unknown ground-truth label \( Y_i\in[K],\) is treated as a latent variable. Let \( (\pi_k)_{k\in[K]}\) denote the ground-truth class distribution, where \(\pi_k:=\mathbb P(Y=k)\).
Here the \Method parameters that one need to estimate are:
\[
\theta
:=
\left(
\mu_\alpha,\mu_\beta,G,H
\right),
\]
where
\(
\mu_\alpha,\mu_\beta\in\mathbb R^K\) and 
\(
G=(g_{r,k})_{r,k}\in\mathbb R^{m\times K}\) and \(
H=(h_{i,k})_{i,k}\in\mathbb R^{n\times K}.\) 

Where $m$ denotes the number of annotators, $n$ the number of items and $K$ the number of classes. 

Conditional on \(Y_i=k\), the annotation likelihood is
\[
\mathbb P_\theta(Y_{ri}=y\mid Y_i=k)
=
p_{ir,k}^{\mathbf 1\{y=k\}}
\left(
\frac{1-p_{ir,k}}{K-1}
\right)^{\mathbf 1\{y\neq k\}}.
\]
Let
\(
\ell_{ir,k}(\theta)
:=
\log
\mathbb P_\theta(Y_{ri}\mid Y_i=k).
\)

We estimate \(\theta\) using an EM algorithm applied to a penalized complete-
data log-likelihood, as it has been done in \cite{whitehill2009glad}. The latent variables handled by the E-step are the
unknown labels \(Y_1,\ldots,Y_n\); the quantities
\(\mu_\alpha,\mu_\beta,G\), and \(H\) are optimized in the M-step.

\paragraph{E-step.}

Given the parameters at step $s$, 
\(\theta^{(s)}\), the E-step computes the latent posterior:
\[
q_{ik}^{(s)}
:=
\mathbb P_{\theta^{(s)}}
\left(
Y_i=k\mid(Y_{r,i})_r
\right).
\]
\paragraph{M-step.}
For \(\xi\in\{\alpha,\beta\}\), define
\[
\lambda_{\xi,\mu}:=(2\sigma_{\xi,\mu}^{2})^{-1},
\qquad
\lambda_{\xi,k}:=(2\sigma_{\xi,k}^{2})^{-1},
\quad k\in[K],
\]
and let
\(
\Lambda_\xi
:=
\operatorname{diag}(\lambda_{\xi,1},\ldots,\lambda_{\xi,K})
\).
Define also
\[
\mathcal R(\theta):
=
\lambda_{\alpha,\mu}\|\mu_\alpha\|_2^2
+
\lambda_{\beta,\mu}\|\mu_\beta\|_2^2
+
\|G\Lambda_\alpha^{1/2}\|_F^2
+
\|H\Lambda_\beta^{1/2}\|_F^2.
\]

A direct calculation shows that at iteration $s$:
\begin{equation}
\label{eq:Q-map}
\begin{aligned}
Q_s(\theta)
&:=
\mathbb E_{Y\mid y;\,\theta^{(s)}}
\left[
\log p_\theta(y,Y)
\right]
-
\mathcal R(\theta)\\
&\phantom{:}=
\sum_{i=1}^{n}
\sum_{k=0}^{K-1}
q_{ik}^{(s)}
\left[
\log\pi_k
+
\sum_{r\in\Omega_i}
\ell_{ir,k}(\theta)
\right]
-
\mathcal R(\theta).
\end{aligned}
\end{equation}
The parameter update is therefore
\begin{equation}\label{Moptim}
 \theta^{(s+1)}
\in
\operatorname*{arg\,max}_{\theta\in\mathcal C}
Q_s(\theta),
\end{equation}
 where
\[
\mathcal C
:=
\left\{
(\mu_\alpha,\mu_\beta,G,H):
\sum_{r=1}^{m}g_{r,k}=
\sum_{i=1}^{n}h_{i,k}=0,\;
k\in[K]
\right\}.
\]

\begin{boxedproposition}
    \label{prop:unique-m-step}
Let
\(\{q_{ik}^{(s)}\}_{i,k}\) be the fixed posterior latent and fixed class proportions
\(\pi_k\), the \(M\)-step problem defined in
Eq.~\eqref{Moptim} admits a unique
maximizer \(\theta^{(s+1)}\) over \(\mathcal C\).
\end{boxedproposition}

The proof relies mainly on the following Theorem, the interested reader can find a proof in \cite{allaire}: 
\begin{theorem}
\label{thm:app-coercive-minimization}
Let \(C\subseteq\mathbb R^d\) be a nonempty, closed, and convex set, and let
\(f:C\to\mathbb R\). If \(f\) is convex and coercive on \(C\), then \(f\)
attains its minimum on \(C\).
\end{theorem}

\begin{lemma}\label{closedconvexeC}
    Let $C$ be the set defined as 
    \[
\left\{
(\mu_\alpha,\mu_\beta,G,H):
\sum_{r=1}^{
m} g_{r,k}=
\sum_{i=1}^{n} h_{i,k}=0,\
k\in[K]
\right\}.
\]
$C$ is closed and convex. 
\end{lemma}
\begin{proofbox}
    \begin{proof}
Let
\[
u=(\mu_\alpha,\mu_\beta,G,H)\in\mathcal C,
\qquad
v=(\mu_\alpha',\mu_\beta',G',H')\in\mathcal C.
\]
For any \(\lambda\in[0,1]\) and every \(k\in[K]\),
\[
\sum_{r=1}^{m}
\left(
\lambda g_{r,k}+(1-\lambda)g'_{r,k}
\right)
=
\lambda\sum_{r=1}^{m}g_{r,k}
+
(1-\lambda)\sum_{r=1}^{m}g'_{r,k}
=
0,
\]
and similarly,
\[
\sum_{i=1}^{n}
\left(
\lambda h_{i,k}+(1-\lambda)h'_{i,k}
\right)
=
0.
\]
It follows that \(\lambda u+(1-\lambda)v\in\mathcal C\)

To prove closedness, let \((u_j)_{j\geq 1}\subseteq\mathcal C\) such that
\(u_j\to u\). Writing
\[
u_j=(\mu_{\alpha,j},\mu_{\beta,j},G_j,H_j),
\]
we have, for every \(j\) and every \(k\in[K]\),
\[
\sum_{r=1}^{m}g_{r,k,j}=0,
\qquad
\sum_{i=1}^{n}h_{i,k,j}=0.
\]
Because these are finite sums, we may pass to the limit:
\[
\sum_{r=1}^{m}g_{r,k}
=
\lim_{j\to\infty}
\sum_{r=1}^{m}g_{r,k,j}
=
0,
\]
and similarly
\[
\sum_{i=1}^{n}h_{i,k}=0.
\]
Therefore \(u\in\mathcal C\), and \(\mathcal C\) is closed.
\end{proof}
\end{proofbox}
The preceding lemma establishes that the feasible set \(\mathcal C\) is closed
and convex; it is also clearly nonempty, since it contains the zero parameter
vector. We can therefore use Theorem~\ref{thm:app-coercive-minimization} to
study the \(M\)-step. More precisely, maximizing \(Q_s\) over \(\mathcal C\) is
equivalent to minimizing
\[
f_s:=-Q_s
\]
over the same set. It remains to show that \(f_s\) is coercive and strictly
convex. Coercivity will guarantee the existence of a minimizer, while strict
convexity will ensure its uniqueness.

\begin{proofbox}
\begin{proof}
First, \(f_s\) is convex. For fixed posterior weights
\(\{q_{ik}^{(s)}\}_{i,k}\), the Hessian of \(Q_s\) with respect to
\(\theta=(\mu_\alpha,\mu_\beta,G,H)\) is
\[
\nabla_{\theta}^{2}Q_s
=
-
\sum_{i=1}^{n}
\sum_{k\in[K]}
q_{ik}^{(s)}
\sum_{r\in\Omega_i}
p_{ir,k}(1-p_{ir,k})
\,x_{ir,k}x_{ir,k}^{\top}
-
\nabla_{\theta}^{2}\mathcal R,
\]
where
\[
x_{ir,k}:=\nabla_\theta\eta_{ir,k}.
\]
Since \(p_{ir,k}(1-p_{ir,k})\geq 0\) and
\(x_{ir,k}x_{ir,k}^{\top}\) is symmetric positive semidefinite. Moreover $\mathcal{R}$ is strictly convex. Hence \(Q_s\) is strictly concave and
\(f_s=-Q_s\) is strictly convex.

Next, \(f_s\) is coercive. Since
\[
\log\pi_k\leq 0
\qquad\text{and}\qquad
\ell_{ir,k}(\theta)\leq 0,
\]
we have that:
\begin{equation}\label{RQ}
    -Q_s(\theta)\geq \mathcal R(\theta).
\end{equation}

\[
\mathcal R(\theta)\longrightarrow+\infty
\qquad\text{as}\qquad
\|\theta\|\longrightarrow\infty.
\]
Therefore using: \eqref{RQ} it follows that:
\[
-Q_s(\theta)\longrightarrow+\infty
\qquad\text{as}\qquad
\|\theta\|\longrightarrow\infty,
\]
so \(f_s\) is coercive.

Finally, \(\mathcal C\) is nonempty, closed, and convex by
Lemma~\ref{closedconvexeC}. Thus \(f_s\) attains a minimum on
\(\mathcal C\), and strict convexity guarantees that this minimizer is unique.
Equivalently, \(Q_s\) admits a unique maximizer on \(\mathcal C\).
\end{proof}
\end{proofbox}

\section{Experimental Evaluation}\label{sec:exps}
\label{sec:experiments}
Our experimental analysis is made of three 
parts. We first evaluate aggregation methods on a controlled synthetic crowdsourcing
benchmark designed to reproduce two sources of imbalance: class imbalance in the
ground-truth labels and class-dependent imbalance in annotator expertise.
Then we evaluate our \Method model across $27$ real-world datasets. 
Finally, we analyze how Gold Majority Vote \citep{Ustalov_2024} evolves as the
number of available gold labels increases, and compare it performances on recall with \Method.
In what follows, we define the class-imbalance ratio as
\[
\mathrm{Imb}
:=
\frac{
    \max_{k\in[K]} \mathbb P(Y=k)
}{
    \min_{k\in[K]} \mathbb P(Y=k)
}.
\]
A value of \(\mathrm{Imb}=1\) corresponds to a balanced class distribution,
while a large value indicate a stronge class imbalance.

\subsection{Synthetic Data}
\label{subsec:synthetic_data}

\paragraph{Dataset.} 
The synthetic dataset contains $3{,}000$ items, with an imbalanced ratio of $\text{Imb}\approx 7 $ belonging to the minority class. Each item is annotated by $8$ annotators. Annotators are
partitioned into $4$ accuracy profiles: majority-class specialists, minority-class specialists, annotators accurate on both classes, and annotators inaccurate on both classes. Items are assigned either an easy or a
hard difficulty level. 
More formally: 
For each item \(i\), we sample a ground-truth label
\[
Y_i \sim \operatorname{Bernoulli}(\pi),
\]
where class \(1\) is the minority class, and an item-difficulty indicator
\[
H_i \sim \operatorname{Bernoulli}(\rho).
\]
Each annotator \(r\) belongs to a type
\(T_r\in\{\mathrm{good},\mathrm{maj},\mathrm{bad},\mathrm{min}\}\),
with fixed type proportions. For every observed annotation \((i,r)\),
\[
\mathbb P(Y_{ir}=Y_i\mid Y_i=k,H_i=h,T_r=t)
=
q_{t,k}-h\delta,
\]
where \(q_{t,k}\) is the class-conditional reliability of annotator type \(t\)
and \(\delta\) is the hard-item penalty. Otherwise, the annotator outputs
\(1-Y_i\). Each item receives
\[
N_i=1+\operatorname{Poisson}(\lambda-1)
\]
annotations, with annotators sampled uniformly without replacement.
In the controlled sweep, the proportions of reliable and unreliable
annotators are fixed, while the mass of majority and minority specialists is
redistributed between the two groups.
\paragraph{Robustness to fewer minority specialists.}
We study the robustness of aggregation methods when the number of
minority-class specialists (see Eq.~\eqref{def:minority-specialist}) decreases while the number of majority-class
specialists increases.

\paragraph{Results.} Figure~\ref{fig:minority-recall-fixed-crowd} shows that
minority recall deteriorates for standard aggregation methods as minority
expertise becomes scarce, while \Method remains more stable.

\begin{figure}[!t]
    \centering
    \includegraphics[width=1\columnwidth]{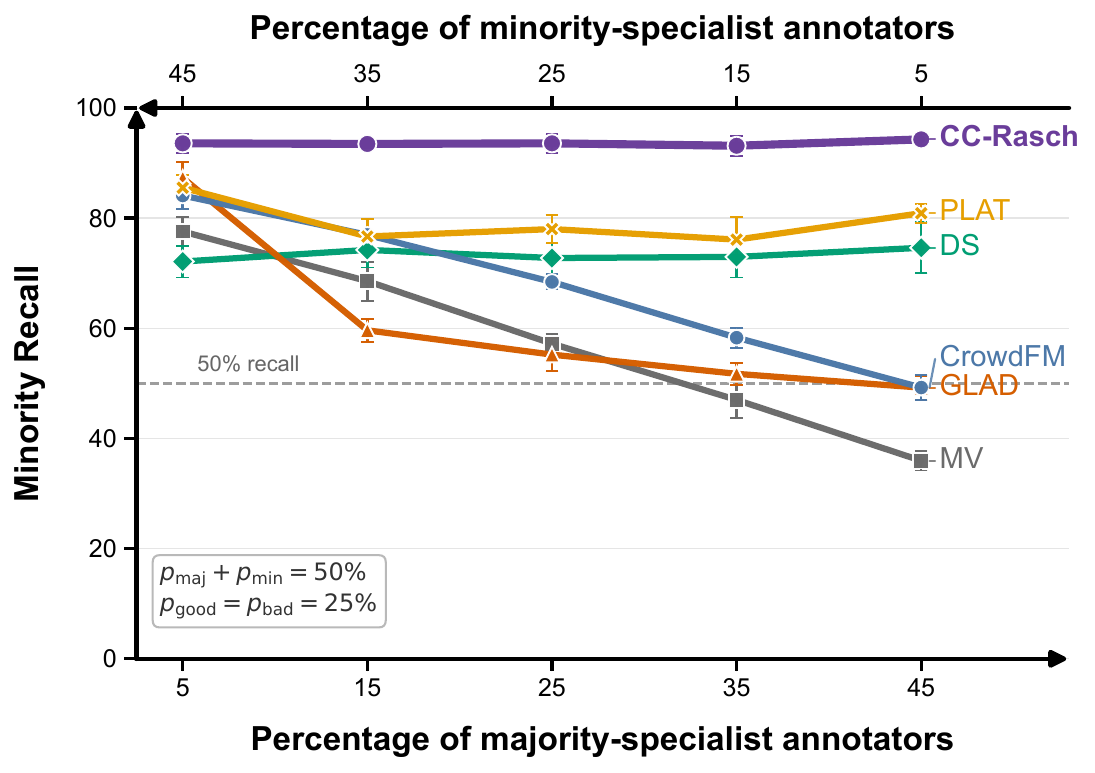}
    \caption{
Minority-class recall as the proportion of majority-class specialists
\(p_{\mathrm{maj}}\) increases and the proportion of minority-class specialists
\(p_\mathrm{min}\) decreases.
Their total proportion is fixed at
\(p_{\mathrm{maj}}+p_{\mathrm{\min}}=50\%\), while
\(p_{\mathrm{good}}=p_{\mathrm{bad}}=25\%\).
Points report mean minority recall and error bars represent 95\% confidence
intervals.
}
    \label{fig:minority-recall-fixed-crowd}
\end{figure}
\subsection{Real World datasets}

\begin{table*}[!t]
\centering
\renewcommand{\arraystretch}{1.25}
\setlength{\tabcolsep}{3pt}

\resizebox{\textwidth}{!}{%
\begin{tabular}{lcccccccc}
\toprule
\textbf{Method}
& \multicolumn{2}{c}{\textbf{All}}
& \multicolumn{2}{c}{\textbf{Imb. $\geq 3$}}
& \multicolumn{2}{c}{\textbf{Ann./item $\geq 10$}}
& \multicolumn{2}{c}{\textbf{Ann./item $<5$}}
\\
\cmidrule(lr){2-3}
\cmidrule(lr){4-5}
\cmidrule(lr){6-7}
\cmidrule(lr){8-9}
& \textbf{Recall $\uparrow$}
& \textbf{Bal.-Acc. $\uparrow$}
& \textbf{Recall $\uparrow$}
& \textbf{Bal.-Acc. $\uparrow$}
& \textbf{Recall $\uparrow$}
& \textbf{Bal.-Acc. $\uparrow$}
& \textbf{Recall $\uparrow$}
& \textbf{Bal.-Acc. $\uparrow$}
\\
\midrule

Majority Vote
& $0.732 \pm 0.209$
& $0.809 \pm 0.126$
& $0.732 \pm 0.138$
& $0.787 \pm 0.117$
& $0.726 \pm 0.233$
& $0.801 \pm 0.140$
& $0.686 \pm 0.206$
& $0.786 \pm 0.111$
\\

Dawid--Skene~\citeyearpar{dawid1979maximum}
& $0.811 \pm 0.204$
& \second{0.844 \pm 0.123}
& $0.803 \pm 0.130$
& \best{0.814 \pm 0.102}
& $0.830 \pm 0.233$
& \second{0.851 \pm 0.142}
& $0.681 \pm 0.184$
& $0.785 \pm 0.093$
\\

GLAD~\citeyearpar{whitehill2009glad}
& $0.720 \pm 0.271$
& $0.812 \pm 0.140$
& $0.776 \pm 0.161$
& \second{0.806 \pm 0.111}
& $0.732 \pm 0.278$
& $0.813 \pm 0.150$
& $0.524 \pm 0.338$
& $0.732 \pm 0.150$
\\

CrowdFM~\citeyearpar{crowdfm}
& $0.739 \pm 0.211$
& $0.817 \pm 0.126$
& $0.734 \pm 0.144$
& $0.791 \pm 0.107$
& $0.754 \pm 0.232$
& $0.819 \pm 0.138$
& $0.621 \pm 0.204$
& $0.762 \pm 0.109$
\\

PLAT
& \best{0.859 \pm 0.199}
& $0.821 \pm 0.125$
& \second{0.881 \pm 0.118}
& $0.781 \pm 0.116$
& \second{0.840 \pm 0.247}
& $0.818 \pm 0.142$
& \best{0.873 \pm 0.153}
& \second{0.792 \pm 0.112}
\\

\textcolor{darkblue}{\textbf{\Method}} (OURS)
& \second{0.859 \pm 0.190}
& \best{0.847 \pm 0.127}
& \best{0.888 \pm 0.064}
& $0.802 \pm 0.106$
& \best{0.851 \pm 0.239}
& \best{0.853 \pm 0.149}
& \second{0.815 \pm 0.107}
& \best{0.793 \pm 0.097}
\\

\bottomrule
\end{tabular}%
}

\caption{
Performance on  binary datasets.
}
\label{tab:regimes_binary_recall_balacc_with_plat}
\end{table*}

\begin{table*}[!t]
\centering
\small
\renewcommand{\arraystretch}{1.25}
\setlength{\tabcolsep}{3pt}
\resizebox{\textwidth}{!}{%
\begin{tabular}{lcccccccc}
\toprule
\textbf{Method}
& \multicolumn{2}{c}{\textbf{All}}
& \multicolumn{2}{c}{\textbf{Imb $\geq 3$}}
& \multicolumn{2}{c}{\textbf{Ann/item $\geq 10$}}
& \multicolumn{2}{c}{\textbf{Ann/item $<5$}} \\
\cmidrule(lr){2-3}
\cmidrule(lr){4-5}
\cmidrule(lr){6-7}
\cmidrule(lr){8-9}
& \textbf{Recall $\uparrow$}
& \textbf{Bal-Acc $\uparrow$}
& \textbf{Recall $\uparrow$}
& \textbf{Bal-Acc $\uparrow$}
& \textbf{Recall $\uparrow$}
& \textbf{Bal-Acc $\uparrow$}
& \textbf{Recall $\uparrow$}
& \textbf{Bal-Acc $\uparrow$} \\
\midrule

Majority Vote 
& $0.659 \pm 0.341$
& $0.740 \pm 0.119$
& $0.185 \pm 0.128$
& \second{0.745 \pm 0.043}
& $0.520 \pm 0.476$
& $0.812 \pm 0.128$
& $0.831 \pm 0.143$
& $0.721 \pm 0.048$ \\

Dawid--Skene~\citeyearpar{dawid1979maximum}
& \second{0.691 \pm 0.316}
& \second{0.765 \pm 0.109}
& \second{0.296 \pm 0.128}
& $0.739 \pm 0.010$
& \best{0.582 \pm 0.418}
& \best{0.826 \pm 0.123}
& \best{0.903 \pm 0.093}
& \best{0.772 \pm 0.021} \\

GLAD~\citeyearpar{whitehill2009glad}
& $0.675 \pm 0.344$
& $0.760 \pm 0.112$
& $0.185 \pm 0.128$
& $0.743 \pm 0.031$
& $0.520 \pm 0.476$
& $0.819 \pm 0.126$
& $0.877 \pm 0.108$
& $0.761 \pm 0.040$ \\

CrowdFM~\citeyearpar{crowdfm}
& $0.656 \pm 0.375$
& $0.758 \pm 0.121$
& $0.139 \pm 0.192$
& $0.738 \pm 0.039$
& $0.479 \pm 0.524$
& $0.811 \pm 0.130$
& \second{0.881 \pm 0.102}
& $0.762 \pm 0.044$ \\

\textcolor{darkblue}{\textbf{\Method}} (OURS)
& \best{0.722 \pm 0.295}
& \best{0.770 \pm 0.115}
& \best{0.370 \pm 0.257}
& \best{0.754 \pm 0.037}
& \second{0.576 \pm 0.412}
& \second{0.823 \pm 0.122}
& $0.871 \pm 0.138$
& \second{0.771 \pm 0.029} \\

\bottomrule
\end{tabular}%
}
\caption{Performance on multiclass datasets.}
\label{tab:regimes_multiclass_recall_balacc}
\end{table*}

\begin{table*}[!t]
\centering
\scriptsize
\renewcommand{\arraystretch}{1.35}
\setlength{\tabcolsep}{3pt}

\begin{tabular*}{\textwidth}{@{\extracolsep{\fill}}lcccccc}
\toprule
\textbf{Method}
& \multicolumn{3}{c}{\textbf{Large-scale annotations}}
& \multicolumn{3}{c}{\textbf{Large-scale items}} \\
\cmidrule(lr){2-4}
\cmidrule(lr){5-7}

& \textbf{Recall $\uparrow$}
& \textbf{Bal-Acc $\uparrow$}
& \textbf{F1 $\uparrow$}
& \textbf{Recall $\uparrow$}
& \textbf{Bal-Acc $\uparrow$}
& \textbf{F1 $\uparrow$} \\
\midrule

Majority Vote 
& $0.712 \pm 0.206$
& $0.796 \pm 0.158$
& $0.798 \pm 0.158$
& \second{0.747 \pm 0.235}
& $0.794 \pm 0.204$
& $0.794 \pm 0.200$ \\

Dawid--Skene~\citeyearpar{dawid1979maximum}
& \second{0.766 \pm 0.237}
& \second{0.835 \pm 0.146}
& \best{0.831 \pm 0.149}
& $0.723 \pm 0.302$
& \best{0.822 \pm 0.184}
& \best{0.830 \pm 0.185} \\

GLAD~\citeyearpar{whitehill2009glad}
& $0.699 \pm 0.306$
& $0.802 \pm 0.175$
& $0.793 \pm 0.198$
& $0.740 \pm 0.252$
& $0.806 \pm 0.190$
& \second{0.818 \pm 0.186} \\

CrowdFM~\citeyearpar{crowdfm}
& $0.722 \pm 0.223$
& $0.805 \pm 0.162$
& $0.804 \pm 0.163$
& $0.724 \pm 0.257$
& $0.794 \pm 0.203$
& $0.797 \pm 0.199$ \\

\textcolor{darkblue}{\textbf{\Method}} (OURS)
& \best{0.834 \pm 0.176}
& \best{0.837 \pm 0.148}
& \second{0.812 \pm 0.156}
& \best{0.803 \pm 0.234}
& \second{0.820 \pm 0.191}
& $0.793 \pm 0.196$ \\

\bottomrule
\end{tabular*}

\caption{Performance on large-scale datasets.}
\label{tab:large_scale_minority_recall}
\end{table*}

\paragraph{Datasets.}
We evaluate \Method on a broad collection of real-world
crowdsourcing datasets covering natural language processing, computer
vision, and general classification tasks. The natural language processing benchmarks include Recognizing Textual
Entailment (RTE), TREC and NIST--TREC relevance judgments, product
classification, web relevance assessment, HITSpam, and several sentiment
analysis datasets, including PosSent, Weather Sentiment--AMT, and
Sentiment Popularity--AMT \citep{snow2008cheap}. We also include the
Movie Reviews 4-class dataset and three variants of the Medical
CrowdTruth benchmark: Medical CrowdTruth All, Cause, and Treat. We additionally consider common benchmarks such as CF, CF$^\star$, MS, SP, and the ZC variants (ZC-all, ZC-in, and
ZC-us) \citep{crowdfm}. For computer vision, we include Bird, Dog, Face,
LabelMe, and CIFAR-10H, which have been widely used to study annotator
expertise, item difficulty, and human uncertainty in visual
crowdsourcing
\citep{whitehill2009glad,welinder2010multidimensional,rodrigues2013,
russell2008labelme,peterson2019humanuncertaintymakesclassification}. We also consider the challenging Duchenne smile dataset, referred to as the WSCM dataset, which was used in \citet{whitehill2009glad}.

Overall, the benchmark contains both binary and multiclass tasks and
covers a wide range of class-imbalance levels, crowd sizes, dataset
sizes, and annotation densities. This diversity allows us to evaluate
the aggregation methods in both sparse and highly redundant annotation
regimes, as well as on large-scale datasets in terms of either the
number of items or the number of annotations.

\paragraph{Methodology.}
We compare \Method with five aggregation baselines: PLAT, specifically designed
for imbalanced label aggregation, and four widely used methods, namely majority
voting (MV), Dawid--Skene (DS), GLAD, and CrowdFM, a recent deep
learning-based approach. CCGPFL \citep{gilgonzalez2026focal} is not included
because it requires item features to model instance-dependent annotator
reliability. Since such features are unavailable
for many datasets in our benchmark, CCGPFL cannot be evaluated under our common
annotation-only setting.
 We use the official Crowd-Kit implementations of MV, DS,
and GLAD \citep{Ustalov_2024}, and the authors' official implementation of
CrowdFM.

For each method, we evaluate the aggregated labels against the available
ground-truth labels. Values are reported as mean $\pm$ standard deviation across
datasets after averaging each dataset over $4$ random seeds. Let
\(n_k:=\sum_{i=1}^{n}\indic{y_i=k},\)
\(
\mathcal K_{\mathrm{obs}}:=\{k\in[K]:n_k>0\}\)
and let
\(
y_{\min}
\)
denote the least frequent ground-truth class. The empirical recall of class
\(k\in\mathcal K_{\mathrm{obs}}\) is
\[
\widehat R_k
:=
\frac{
\sum_{i=1}^{n}
\indic{y_i=k,\ \widehat y_i=k}
}{
n_k
}.
\]
Minority recall and balanced accuracy are then defined by
\[
\widehat{\mathrm{Recall}}_{\min}
:=
\widehat R_{y_{\min}},
\qquad
\widehat{\mathrm{BalAcc}}
:=
\frac{1}{|\mathcal K_{\mathrm{obs}}|}
\sum_{k\in\mathcal K_{\mathrm{obs}}}
\widehat R_k.
\]
The multiclass
benchmark contains $11$ datasets, including $3$ with imbalanced ratio greater than $3$, $4$ with
$\mathrm{Annotators/item}\geq 10$, and 3 with $\mathrm{Annotators/item}<5$. The binary
benchmark contains 16 datasets, including $4$, $11$, and $2$ datasets in these
respective regimes.
We also consider
large-scale subsets: $8$ datasets with at least $20\,000$ annotations and $4$
datasets with at least $4\,000$ annotated items.

Since high global accuracy may hide poor recovery of rare classes, we treat
minority recall as the primary metric.

We report results separately for
binary datasets and for the full benchmark. In
Table~\ref{tab:large_scale_minority_recall}, we report minority recall
in the large-scale annotation and large-scale item regimes.

\paragraph{Implementation details.}
We implement \Method in Python and optimize its penalized M-step objective
using the L-BFGS-B algorithm through
\texttt{scipy.optimize.minimize} \citep{2020SciPy-NMeth}. At each EM iteration, the optimizer is initialized with the parameter estimates obtained at the previous iteration. The prior $(\pi_k)_k$ is initialized uniformly at the initialization of the EM.

\paragraph{Results.}
Across real-world datasets, \Method shows the strongest and most consistent
ability to recover the minority class. On multiclass datasets, (Table~\ref{tab:regimes_multiclass_recall_balacc}), it improves the
average minority recall from \(0.691\) to \(0.722\) and the balanced accuracy
from \(0.765\) to \(0.770\) over the strongest baseline. The gain is larger in
the most imbalanced regime, reaching \(+7.4\) recall points. 
On binary datasets (Table~\ref{tab:regimes_binary_recall_balacc_with_plat}),
\Method achieves the best overall balanced accuracy (\(0.847\)) and the best
minority recall (\(0.859\), tied with PLAT). Notice that \Method has the lowest variance across datasets compared with PLAT.  
It is also more stable across datasets, with the lowest recall variability
overall and in the imbalanced and sparse-annotation regimes.
Under strong imbalance, \Method reaches \(0.888\) recall, while for
Ann./item~\(\geq 10\), it obtains the best recall (\(0.851\)) and balanced
accuracy (\(0.853\)).

\subsection{\Method versus Gold-Supervised Methods}
\label{subsec:gold_labels}

We investigate whether access to ground-truth labels is necessary for accurate
minority-class recovery. We compare \Method, which is fitted using only the
observed crowd labels and never accesses gold labels, with supervised
aggregation methods whose annotator parameters are estimated from an increasing
proportion of gold-labeled items.

This kind of question has previously been investigated, for instance, by
\citet{Jung_Lease_2015}. We revisit it from the perspective of class imbalance,
with a particular focus on minority recall. Notice that standard majority voting
does not use gold labels and is therefore unaffected by their availability. In
contrast, gold-supervised methods such as Gold Majority Vote and Dawid--Skene
estimate annotator reliability from the available gold-labeled items, so their
performance may improve as the amount of supervision increases.
\paragraph{Methodology}
Let \(\mathcal I=\{1,\ldots,n\}\) be the set of items and let
a real number \(x\in ]0,1[\). We sample a subset
\(\mathcal I_x\subset\mathcal I\), containing \(\lfloor xn\rfloor\) items, for
which the ground-truth labels are assumed to be available, these are the \emph{gold labels}.
We define the gold training set as
\[
    \mathcal D_x^{\mathrm{gold}}
    :=
    \left\{
    \left(i, \left(Y_{ir}\right)_{r\in\Omega_i},Y_i\right):
    i\in\mathcal I_x
    \right\},
\]
and the remaining test set as the complementary of \(\mathcal D_x^{\mathrm{gold}}\). 

Given \(\mathcal D_x^{\mathrm{gold}}\), we fit Gold Majority Vote
\citep{Ustalov_2024} and a gold-supervised Dawid--Skene model.
For Dawid--Skene, the annotator confusion matrices are estimated from the gold
items and then kept fixed when aggregating the annotations of the test items.
In parallel, \Method is fitted classically using only
\(
    \left\{
    \left(i,\left(Y_{ir}\right)_{r\in\Omega_i}\right):
    i\in\mathcal I\setminus\mathcal I_x
    \right\}.
\) All three methods are evaluated on the same (not observed) items.
\begin{figure}[!t]
    \centering
    \includegraphics[width=1\linewidth]{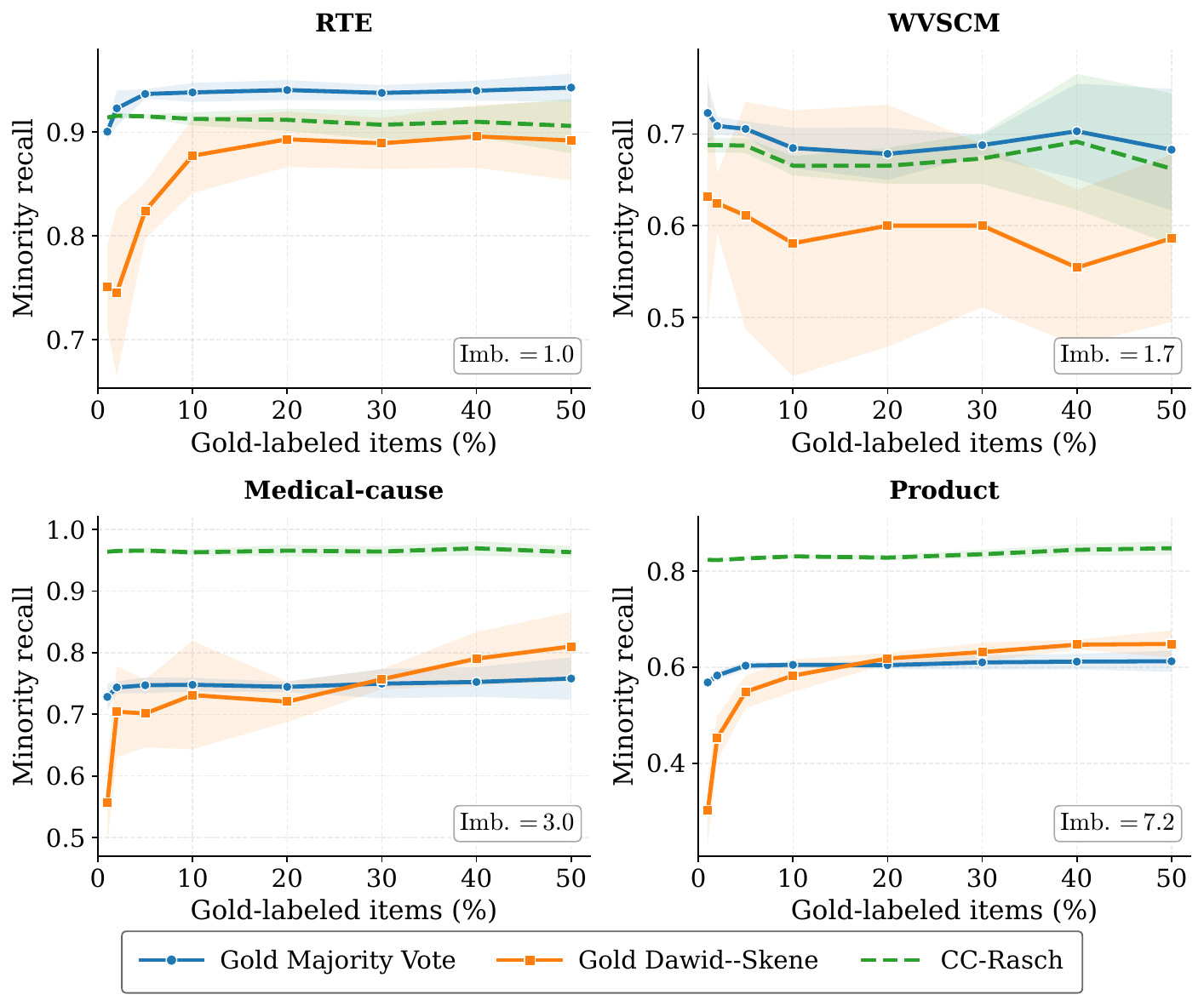}
    \caption{\label{goldlabelfig} Gold label boosted methods versus \Method.}
\end{figure}
\paragraph{Results.}
According to Figure~\ref{goldlabelfig}, 
in the class-balanced case, Gold Majority Vote achieves a minority recall higher
than or close to that of \Method. However, as class imbalance increases, the gap
between \Method and the two gold-supervised methods becomes larger. Gold
Dawid--Skene and Gold Majority Vote obtain similar minority recall on the most
imbalanced datasets, although some differences remain on less imbalanced ones.
This confirms in particular that an accurate estimation of annotator abilities
is not always sufficient under strong class imbalance. Accounting for class-dependent item difficulty allows \Method to better aggregate labels.

\section{Conclusion}\label{sec:conclusions}

In this paper, we studied label aggregation in the class-imbalanced crowdsourcing setting, where rare classes are often the most important ones. We provided a new type of Condorcet Jury Theorem adapted to the imbalanced setting. To address the limitations created by the imbalanced setting, we introduced \Method, a class-dependent aggregation model that combines annotator competence and item difficulty. Our experiments on synthetic and real-world datasets show that \Method consistently improves minority-class recall while remaining competitive in balanced accuracy and F1-score, with particularly strong gains under severe class imbalance, high annotation redundancy, and large-scale settings. 

\section{Detailed results on each dataset}
\subsection{Synthetic dataset generation}
Here we provide a more detailed description of the generation process for synthetic dataset. 

For each item \(i\), we sample a ground-truth label
\(
Y_i \sim \operatorname{Bernoulli}(\pi),
\)
where class \(1\) is the minority class, and an item-difficulty indicator
\[
H_i \sim \operatorname{Bernoulli}(\rho).
\]
Each annotator \(r\) belongs to a type
\(T_r\in\{\mathrm{good},\mathrm{maj},\mathrm{bad},\mathrm{min}\}\),
with fixed type proportions. For every observed annotation \((i,r)\),
\[
\mathbb P(Y_{ir}=Y_i\mid Y_i=k,H_i=h,T_r=t)
=
q_{t,k}-h\delta,
\]
where \(q_{t,k}\) is the class-conditional reliability of annotator type \(t\)
and \(\delta\) is the hard-item penalty. Otherwise, the annotator outputs
\(1-Y_i\). Each item receives
\[
N_i=1+\operatorname{Poisson}(\lambda-1)
\]
annotations, with annotators sampled uniformly without replacement.
In the controlled sweep, the proportions of reliable and unreliable
annotators are fixed, while the mass of majority and minority specialists is
redistributed between the two groups.
\subsection{Detailed benchmark}
\begin{longtable}{llrrrr}
\caption{\label{tab:all_dataset_results_extended_with_plat}}\\
\toprule
\textbf{Dataset} & \textbf{Method} & \textbf{Imb.} & \textbf{Ann./item} & \textbf{Min. Recall $\uparrow$} & \textbf{Macro F1 $\uparrow$} \\
\midrule
\endfirsthead
\toprule
\textbf{Dataset} & \textbf{Method} & \textbf{Imb.} & \textbf{Ann./item} & \textbf{Min. Recall $\uparrow$} & \textbf{Macro F1 $\uparrow$} \\
\midrule
\endhead
\midrule
\multicolumn{6}{r}{Continued on next page}\\
\endfoot
\bottomrule
\endlastfoot
Loneliness (Older Adults) & Majority Vote & 19.00 & 4.77 & $0.800 \pm 0.000$ & $0.667 \pm 0.000$ \\
 & Dawid--Skene & 19.00 & 4.77 & $0.600 \pm 0.000$ & $0.634 \pm 0.000$ \\
 & GLAD & 19.00 & 4.77 & $0.600 \pm 0.000$ & \textcolor{darkblue}{\textbf{$0.693 \pm 0.000$}} \\
 & CrowdFM & 19.00 & 4.77 & $0.600 \pm 0.000$ & \textcolor{darkblue}{\textbf{$0.693 \pm 0.000$}} \\
 & $\Phi$-Rasch (OUR) & 19.00 & 4.77 & $0.800 \pm 0.000$ & $0.494 \pm 0.000$ \\
 & PLAT & 19.00 & 4.77 & \textcolor{darkblue}{\textbf{$1.000 \pm 0.000$}} & $0.545 \pm 0.000$ \\
\addlinespace[2pt]
CF & Majority Vote & 10.22 & 5.73 & $0.333 \pm 0.000$ & \textcolor{darkblue}{\textbf{$0.810 \pm 0.000$}} \\
 & Dawid--Skene & 10.22 & 5.73 & $0.444 \pm 0.000$ & $0.734 \pm 0.000$ \\
 & GLAD & 10.22 & 5.73 & $0.333 \pm 0.000$ & $0.775 \pm 0.000$ \\
 & CrowdFM & 10.22 & 5.73 & $0.361 \pm 0.106$ & $0.777 \pm 0.016$ \\
 & $\Phi$-Rasch (OUR) & 10.22 & 5.73 & \textcolor{darkblue}{\textbf{$0.667 \pm 0.000$}} & $0.754 \pm 0.000$ \\
\addlinespace[2pt]
CF* & Majority Vote & 10.22 & 20.10 & $0.111 \pm 0.000$ & $0.723 \pm 0.000$ \\
 & Dawid--Skene & 10.22 & 20.10 & \textcolor{darkblue}{\textbf{$0.222 \pm 0.000$}} & $0.727 \pm 0.000$ \\
 & GLAD & 10.22 & 20.10 & $0.111 \pm 0.000$ & \textcolor{darkblue}{\textbf{$0.728 \pm 0.000$}} \\
 & CrowdFM & 10.22 & 20.10 & $0.028 \pm 0.056$ & $0.705 \pm 0.021$ \\
 & $\Phi$-Rasch (OUR) & 10.22 & 20.10 & \textcolor{darkblue}{\textbf{$0.222 \pm 0.000$}} & $0.728 \pm 0.000$ \\
\addlinespace[2pt]
Weather Sentiment - AMT & Majority Vote & 10.22 & 20.00 & $0.111 \pm 0.000$ & $0.723 \pm 0.000$ \\
 & Dawid--Skene & 10.22 & 20.00 & \textcolor{darkblue}{\textbf{$0.222 \pm 0.000$}} & $0.727 \pm 0.000$ \\
 & GLAD & 10.22 & 20.00 & $0.111 \pm 0.000$ & \textcolor{darkblue}{\textbf{$0.728 \pm 0.000$}} \\
 & CrowdFM & 10.22 & 20.00 & $0.028 \pm 0.056$ & $0.705 \pm 0.021$ \\
 & $\Phi$-Rasch (OUR) & 10.22 & 20.00 & \textcolor{darkblue}{\textbf{$0.222 \pm 0.000$}} & $0.725 \pm 0.000$ \\
\addlinespace[2pt]
Amazon Sentiment (Negative) & Majority Vote & 8.70 & 7.61 & $0.827 \pm 0.000$ & $0.906 \pm 0.000$ \\
 & Dawid--Skene & 8.70 & 7.61 & $0.837 \pm 0.000$ & $0.904 \pm 0.000$ \\
 & GLAD & 8.70 & 7.61 & $0.827 \pm 0.000$ & $0.901 \pm 0.000$ \\
 & CrowdFM & 8.70 & 7.61 & $0.837 \pm 0.000$ & \textcolor{darkblue}{\textbf{$0.909 \pm 0.000$}} \\
 & $\Phi$-Rasch (OUR) & 8.70 & 7.61 & \textcolor{darkblue}{\textbf{$0.942 \pm 0.000$}} & $0.762 \pm 0.000$ \\
 & PLAT & 8.70 & 7.61 & $0.904 \pm 0.000$ & $0.857 \pm 0.000$ \\
\addlinespace[2pt]
product & Majority Vote & 7.22 & 3.00 & $0.613 \pm 0.000$ & $0.766 \pm 0.000$ \\
 & Dawid--Skene & 7.22 & 3.00 & $0.640 \pm 0.000$ & \textcolor{darkblue}{\textbf{$0.843 \pm 0.000$}} \\
 & GLAD & 7.22 & 3.00 & $0.518 \pm 0.000$ & $0.799 \pm 0.000$ \\
 & CrowdFM & 7.22 & 3.00 & $0.581 \pm 0.007$ & $0.765 \pm 0.005$ \\
 & $\Phi$-Rasch (OUR) & 7.22 & 3.00 & $0.825 \pm 0.000$ & $0.707 \pm 0.000$ \\
 & PLAT & 7.22 & 3.00 & \textcolor{darkblue}{\textbf{$0.889 \pm 0.000$}} & $0.567 \pm 0.000$ \\
\addlinespace[2pt]
Jigsaw & Majority Vote & 6.13 & 19.93 & $0.916 \pm 0.000$ & $0.975 \pm 0.000$ \\
 & Dawid--Skene & 6.13 & 19.93 & $0.970 \pm 0.000$ & \textcolor{darkblue}{\textbf{$0.982 \pm 0.000$}} \\
 & GLAD & 6.13 & 19.93 & $0.976 \pm 0.000$ & $0.971 \pm 0.000$ \\
 & CrowdFM & 6.13 & 19.93 & $0.930 \pm 0.000$ & $0.972 \pm 0.000$ \\
 & $\Phi$-Rasch (OUR) & 6.13 & 19.93 & $0.959 \pm 0.000$ & $0.918 \pm 0.000$ \\
 & PLAT & 6.13 & 19.93 & \textcolor{darkblue}{\textbf{$1.000 \pm 0.000$}} & $0.743 \pm 0.000$ \\
\addlinespace[2pt]
ZCall & Majority Vote & 3.60 & 10.71 & $0.605 \pm 0.000$ & $0.758 \pm 0.000$ \\
 & Dawid--Skene & 3.60 & 10.71 & $0.862 \pm 0.000$ & $0.758 \pm 0.000$ \\
 & GLAD & 3.60 & 10.71 & $0.824 \pm 0.000$ & \textcolor{darkblue}{\textbf{$0.782 \pm 0.000$}} \\
 & CrowdFM & 3.60 & 10.71 & $0.729 \pm 0.027$ & $0.771 \pm 0.005$ \\
 & $\Phi$-Rasch (OUR) & 3.60 & 10.71 & $0.887 \pm 0.000$ & $0.739 \pm 0.000$ \\
 & PLAT & 3.60 & 10.71 & \textcolor{darkblue}{\textbf{$0.907 \pm 0.000$}} & $0.655 \pm 0.000$ \\
\addlinespace[2pt]
ZCin & Majority Vote & 3.60 & 5.49 & $0.555 \pm 0.000$ & $0.658 \pm 0.000$ \\
 & Dawid--Skene & 3.60 & 5.49 & $0.747 \pm 0.000$ & $0.701 \pm 0.000$ \\
 & GLAD & 3.60 & 5.49 & $0.707 \pm 0.000$ & \textcolor{darkblue}{\textbf{$0.709 \pm 0.000$}} \\
 & CrowdFM & 3.60 & 5.49 & $0.597 \pm 0.023$ & $0.661 \pm 0.001$ \\
 & $\Phi$-Rasch (OUR) & 3.60 & 5.49 & \textcolor{darkblue}{\textbf{$0.828 \pm 0.000$}} & $0.674 \pm 0.000$ \\
 & PLAT & 3.60 & 5.49 & \textcolor{darkblue}{\textbf{$0.828 \pm 0.000$}} & $0.614 \pm 0.000$ \\
\addlinespace[2pt]
ZCus & Majority Vote & 3.60 & 5.98 & $0.666 \pm 0.000$ & $0.801 \pm 0.000$ \\
 & Dawid--Skene & 3.60 & 5.98 & $0.853 \pm 0.000$ & $0.784 \pm 0.000$ \\
 & GLAD & 3.60 & 5.98 & $0.797 \pm 0.000$ & \textcolor{darkblue}{\textbf{$0.830 \pm 0.000$}} \\
 & CrowdFM & 3.60 & 5.98 & $0.686 \pm 0.027$ & $0.805 \pm 0.003$ \\
 & $\Phi$-Rasch (OUR) & 3.60 & 5.98 & \textcolor{darkblue}{\textbf{$0.905 \pm 0.000$}} & $0.725 \pm 0.000$ \\
 & PLAT & 3.60 & 5.98 & $0.894 \pm 0.000$ & $0.700 \pm 0.000$ \\
\addlinespace[2pt]
Loneliness (Intervention) & Majority Vote & 3.17 & 14.36 & $0.875 \pm 0.000$ & $0.450 \pm 0.000$ \\
 & Dawid--Skene & 3.17 & 14.36 & $0.917 \pm 0.000$ & \textcolor{darkblue}{\textbf{$0.508 \pm 0.000$}} \\
 & GLAD & 3.17 & 14.36 & \textcolor{darkblue}{\textbf{$0.958 \pm 0.000$}} & $0.440 \pm 0.000$ \\
 & CrowdFM & 3.17 & 14.36 & $0.917 \pm 0.000$ & $0.489 \pm 0.000$ \\
 & $\Phi$-Rasch (OUR) & 3.17 & 14.36 & \textcolor{darkblue}{\textbf{$0.958 \pm 0.000$}} & $0.429 \pm 0.000$ \\
 & PLAT & 3.17 & 14.36 & $0.625 \pm 0.000$ & $0.496 \pm 0.000$ \\
\addlinespace[2pt]
medical crowdtruth cause & Majority Vote & 2.97 & 13.42 & $0.720 \pm 0.000$ & $0.863 \pm 0.000$ \\
 & Dawid--Skene & 2.97 & 13.42 & $0.904 \pm 0.000$ & \textcolor{darkblue}{\textbf{$0.900 \pm 0.000$}} \\
 & GLAD & 2.97 & 13.42 & $0.745 \pm 0.000$ & $0.874 \pm 0.000$ \\
 & CrowdFM & 2.97 & 13.42 & $0.749 \pm 0.021$ & $0.876 \pm 0.006$ \\
 & $\Phi$-Rasch (OUR) & 2.97 & 13.42 & $0.971 \pm 0.000$ & $0.883 \pm 0.000$ \\
 & PLAT & 2.97 & 13.42 & \textcolor{darkblue}{\textbf{$0.996 \pm 0.000$}} & $0.746 \pm 0.000$ \\
\addlinespace[2pt]
movie reviews 4class & Majority Vote & 2.95 & 4.96 & $0.967 \pm 0.000$ & $0.645 \pm 0.000$ \\
 & Dawid--Skene & 2.95 & 4.96 & \textcolor{darkblue}{\textbf{$0.995 \pm 0.000$}} & $0.709 \pm 0.000$ \\
 & GLAD & 2.95 & 4.96 & $0.978 \pm 0.000$ & $0.687 \pm 0.000$ \\
 & CrowdFM & 2.95 & 4.96 & $0.969 \pm 0.005$ & $0.691 \pm 0.020$ \\
 & $\Phi$-Rasch (OUR) & 2.95 & 4.96 & $0.989 \pm 0.000$ & \textcolor{darkblue}{\textbf{$0.710 \pm 0.000$}} \\
\addlinespace[2pt]
hitspam crowdflower & Majority Vote & 2.23 & 22.97 & \textcolor{darkblue}{\textbf{$0.097 \pm 0.000$}} & \textcolor{darkblue}{\textbf{$0.469 \pm 0.000$}} \\
 & Dawid--Skene & 2.23 & 22.97 & \textcolor{darkblue}{\textbf{$0.097 \pm 0.000$}} & \textcolor{darkblue}{\textbf{$0.469 \pm 0.000$}} \\
 & GLAD & 2.23 & 22.97 & \textcolor{darkblue}{\textbf{$0.097 \pm 0.000$}} & \textcolor{darkblue}{\textbf{$0.469 \pm 0.000$}} \\
 & CrowdFM & 2.23 & 22.97 & \textcolor{darkblue}{\textbf{$0.097 \pm 0.000$}} & \textcolor{darkblue}{\textbf{$0.469 \pm 0.000$}} \\
 & $\Phi$-Rasch (OUR) & 2.23 & 22.97 & \textcolor{darkblue}{\textbf{$0.097 \pm 0.000$}} & \textcolor{darkblue}{\textbf{$0.469 \pm 0.000$}} \\
 & PLAT & 2.23 & 22.97 & \textcolor{darkblue}{\textbf{$0.097 \pm 0.000$}} & \textcolor{darkblue}{\textbf{$0.469 \pm 0.000$}} \\
\addlinespace[2pt]
Web & Majority Vote & 2.09 & 5.84 & $0.918 \pm 0.000$ & $0.773 \pm 0.000$ \\
 & Dawid--Skene & 2.09 & 5.84 & $0.888 \pm 0.000$ & $0.826 \pm 0.000$ \\
 & GLAD & 2.09 & 5.84 & $0.924 \pm 0.000$ & $0.826 \pm 0.000$ \\
 & CrowdFM & 2.09 & 5.84 & $0.928 \pm 0.077$ & $0.856 \pm 0.006$ \\
 & $\Phi$-Rasch (OUR) & 2.09 & 5.84 & \textcolor{darkblue}{\textbf{$0.970 \pm 0.000$}} & \textcolor{darkblue}{\textbf{$0.869 \pm 0.000$}} \\
\addlinespace[2pt]
medical crowdtruth all & Majority Vote & 1.92 & 13.49 & $0.780 \pm 0.000$ & $0.886 \pm 0.000$ \\
 & Dawid--Skene & 1.92 & 13.49 & $0.910 \pm 0.000$ & \textcolor{darkblue}{\textbf{$0.929 \pm 0.000$}} \\
 & GLAD & 1.92 & 13.49 & $0.823 \pm 0.000$ & $0.907 \pm 0.000$ \\
 & CrowdFM & 1.92 & 13.49 & $0.817 \pm 0.018$ & $0.902 \pm 0.006$ \\
 & $\Phi$-Rasch (OUR) & 1.92 & 13.49 & $0.959 \pm 0.000$ & \textcolor{darkblue}{\textbf{$0.929 \pm 0.000$}} \\
 & PLAT & 1.92 & 13.49 & \textcolor{darkblue}{\textbf{$0.994 \pm 0.000$}} & $0.844 \pm 0.000$ \\
\addlinespace[2pt]
wvscm & Majority Vote & 1.74 & 12.26 & \textcolor{darkblue}{\textbf{$0.741 \pm 0.000$}} & $0.714 \pm 0.000$ \\
 & Dawid--Skene & 1.74 & 12.26 & $0.672 \pm 0.000$ & $0.713 \pm 0.000$ \\
 & GLAD & 1.74 & 12.26 & $0.621 \pm 0.000$ & \textcolor{darkblue}{\textbf{$0.759 \pm 0.000$}} \\
 & CrowdFM & 1.74 & 12.26 & $0.698 \pm 0.017$ & $0.711 \pm 0.014$ \\
 & $\Phi$-Rasch (OUR) & 1.74 & 12.26 & $0.690 \pm 0.000$ & $0.721 \pm 0.000$ \\
\addlinespace[2pt]
nist-trec-relevance & Majority Vote & 1.74 & 5.87 & $0.490 \pm 0.000$ & $0.528 \pm 0.000$ \\
 & Dawid--Skene & 1.74 & 5.87 & $0.333 \pm 0.000$ & \textcolor{darkblue}{\textbf{$0.568 \pm 0.000$}} \\
 & GLAD & 1.74 & 5.87 & \textcolor{darkblue}{\textbf{$0.528 \pm 0.000$}} & $0.566 \pm 0.000$ \\
 & CrowdFM & 1.74 & 5.87 & $0.441 \pm 0.044$ & $0.534 \pm 0.007$ \\
 & $\Phi$-Rasch (OUR) & 1.74 & 5.87 & $0.468 \pm 0.000$ & $0.559 \pm 0.000$ \\
\addlinespace[2pt]
LabelMe & Majority Vote & 1.73 & 2.55 & \textcolor{darkblue}{\textbf{$0.843 \pm 0.000$}} & $0.765 \pm 0.000$ \\
 & Dawid--Skene & 1.73 & 2.55 & $0.809 \pm 0.000$ & \textcolor{darkblue}{\textbf{$0.786 \pm 0.000$}} \\
 & GLAD & 1.73 & 2.55 & $0.764 \pm 0.000$ & $0.770 \pm 0.000$ \\
 & CrowdFM & 1.73 & 2.55 & $0.770 \pm 0.030$ & $0.764 \pm 0.006$ \\
 & $\Phi$-Rasch (OUR) & 1.73 & 2.55 & $0.719 \pm 0.000$ & $0.768 \pm 0.000$ \\
\addlinespace[2pt]
Amazon Sentiment (Book) & Majority Vote & 1.57 & 7.63 & $0.980 \pm 0.000$ & $0.960 \pm 0.000$ \\
 & Dawid--Skene & 1.57 & 7.63 & $0.977 \pm 0.000$ & $0.961 \pm 0.000$ \\
 & GLAD & 1.57 & 7.63 & $0.980 \pm 0.000$ & $0.963 \pm 0.000$ \\
 & CrowdFM & 1.57 & 7.63 & $0.982 \pm 0.000$ & $0.962 \pm 0.000$ \\
 & $\Phi$-Rasch (OUR) & 1.57 & 7.63 & $0.980 \pm 0.000$ & \textcolor{darkblue}{\textbf{$0.964 \pm 0.000$}} \\
 & PLAT & 1.57 & 7.63 & \textcolor{darkblue}{\textbf{$0.990 \pm 0.000$}} & $0.942 \pm 0.000$ \\
\addlinespace[2pt]
Loneliness (Technology) & Majority Vote & 1.56 & 4.86 & $0.897 \pm 0.000$ & $0.886 \pm 0.000$ \\
 & Dawid--Skene & 1.56 & 4.86 & \textcolor{darkblue}{\textbf{$0.949 \pm 0.000$}} & $0.887 \pm 0.000$ \\
 & GLAD & 1.56 & 4.86 & $0.897 \pm 0.000$ & $0.876 \pm 0.000$ \\
 & CrowdFM & 1.56 & 4.86 & $0.897 \pm 0.000$ & $0.886 \pm 0.000$ \\
 & $\Phi$-Rasch (OUR) & 1.56 & 4.86 & \textcolor{darkblue}{\textbf{$0.949 \pm 0.000$}} & \textcolor{darkblue}{\textbf{$0.907 \pm 0.000$}} \\
 & PLAT & 1.56 & 4.86 & \textcolor{darkblue}{\textbf{$0.949 \pm 0.000$}} & $0.897 \pm 0.000$ \\
\addlinespace[2pt]
Dog & Majority Vote & 1.35 & 10.00 & $0.860 \pm 0.000$ & $0.816 \pm 0.000$ \\
 & Dawid--Skene & 1.35 & 10.00 & \textcolor{darkblue}{\textbf{$0.884 \pm 0.000$}} & \textcolor{darkblue}{\textbf{$0.845 \pm 0.000$}} \\
 & GLAD & 1.35 & 10.00 & $0.860 \pm 0.000$ & $0.834 \pm 0.000$ \\
 & CrowdFM & 1.35 & 10.00 & $0.860 \pm 0.011$ & $0.820 \pm 0.003$ \\
 & $\Phi$-Rasch (OUR) & 1.35 & 10.00 & $0.866 \pm 0.000$ & $0.837 \pm 0.000$ \\
\addlinespace[2pt]
bird & Majority Vote & 1.28 & 39.00 & $0.553 \pm 0.000$ & $0.738 \pm 0.000$ \\
 & Dawid--Skene & 1.28 & 39.00 & $0.851 \pm 0.000$ & $0.876 \pm 0.000$ \\
 & GLAD & 1.28 & 39.00 & $0.383 \pm 0.000$ & $0.671 \pm 0.000$ \\
 & CrowdFM & 1.28 & 39.00 & $0.601 \pm 0.032$ & $0.776 \pm 0.017$ \\
 & $\Phi$-Rasch (OUR) & 1.28 & 39.00 & $0.872 \pm 0.000$ & \textcolor{darkblue}{\textbf{$0.914 \pm 0.000$}} \\
 & PLAT & 1.25 & 39.00 & \textcolor{darkblue}{\textbf{$0.875 \pm 0.000$}} & $0.851 \pm 0.000$ \\
\addlinespace[2pt]
Trec & Majority Vote & 1.27 & 4.64 & $0.432 \pm 0.000$ & $0.632 \pm 0.000$ \\
 & Dawid--Skene & 1.27 & 4.64 & $0.534 \pm 0.000$ & $0.685 \pm 0.000$ \\
 & GLAD & 1.27 & 4.64 & $0.080 \pm 0.000$ & $0.432 \pm 0.000$ \\
 & CrowdFM & 1.27 & 4.64 & $0.406 \pm 0.033$ & $0.618 \pm 0.012$ \\
 & $\Phi$-Rasch (OUR) & 1.27 & 4.64 & \textcolor{darkblue}{\textbf{$0.688 \pm 0.000$}} & \textcolor{darkblue}{\textbf{$0.700 \pm 0.000$}} \\
 & PLAT & 1.27 & 4.64 & $0.653 \pm 0.000$ & $0.654 \pm 0.000$ \\
\addlinespace[2pt]
Bird & Majority Vote & 1.25 & 39.00 & $0.562 \pm 0.000$ & $0.742 \pm 0.000$ \\
 & Dawid--Skene & 1.25 & 39.00 & $0.854 \pm 0.000$ & $0.887 \pm 0.000$ \\
 & GLAD & 1.25 & 39.00 & $0.396 \pm 0.000$ & $0.678 \pm 0.000$ \\
 & CrowdFM & 1.25 & 39.00 & $0.609 \pm 0.031$ & $0.780 \pm 0.016$ \\
 & $\Phi$-Rasch (OUR) & 1.25 & 39.00 & \textcolor{darkblue}{\textbf{$0.875 \pm 0.000$}} & \textcolor{darkblue}{\textbf{$0.915 \pm 0.000$}} \\
 & PLAT & 1.25 & 39.00 & \textcolor{darkblue}{\textbf{$0.875 \pm 0.000$}} & $0.851 \pm 0.000$ \\
\addlinespace[2pt]
MS & Majority Vote & 1.19 & 4.21 & $0.683 \pm 0.000$ & $0.709 \pm 0.000$ \\
 & Dawid--Skene & 1.19 & 4.21 & \textcolor{darkblue}{\textbf{$0.905 \pm 0.000$}} & $0.770 \pm 0.000$ \\
 & GLAD & 1.19 & 4.21 & $0.889 \pm 0.000$ & $0.793 \pm 0.000$ \\
 & CrowdFM & 1.19 & 4.21 & \textcolor{darkblue}{\textbf{$0.905 \pm 0.000$}} & \textcolor{darkblue}{\textbf{$0.802 \pm 0.004$}} \\
 & $\Phi$-Rasch (OUR) & 1.19 & 4.21 & \textcolor{darkblue}{\textbf{$0.905 \pm 0.000$}} & $0.801 \pm 0.000$ \\
\addlinespace[2pt]
PosSent & Majority Vote & 1.12 & 20.00 & $0.892 \pm 0.000$ & $0.931 \pm 0.000$ \\
 & Dawid--Skene & 1.12 & 20.00 & $0.934 \pm 0.000$ & \textcolor{darkblue}{\textbf{$0.960 \pm 0.000$}} \\
 & GLAD & 1.12 & 20.00 & $0.903 \pm 0.000$ & $0.948 \pm 0.000$ \\
 & CrowdFM & 1.12 & 20.00 & $0.926 \pm 0.005$ & $0.949 \pm 0.003$ \\
 & $\Phi$-Rasch (OUR) & 1.12 & 20.00 & $0.928 \pm 0.000$ & $0.957 \pm 0.000$ \\
 & PLAT & 1.12 & 20.00 & \textcolor{darkblue}{\textbf{$0.956 \pm 0.000$}} & $0.910 \pm 0.000$ \\
\addlinespace[2pt]
medical crowdtruth treat & Majority Vote & 1.07 & 13.61 & $0.829 \pm 0.000$ & $0.899 \pm 0.000$ \\
 & Dawid--Skene & 1.07 & 13.61 & $0.925 \pm 0.000$ & $0.944 \pm 0.000$ \\
 & GLAD & 1.07 & 13.61 & $0.898 \pm 0.000$ & $0.935 \pm 0.000$ \\
 & CrowdFM & 1.07 & 13.61 & $0.898 \pm 0.015$ & $0.935 \pm 0.008$ \\
 & $\Phi$-Rasch (OUR) & 1.07 & 13.61 & $0.962 \pm 0.000$ & \textcolor{darkblue}{\textbf{$0.962 \pm 0.000$}} \\
 & PLAT & 1.07 & 13.61 & \textcolor{darkblue}{\textbf{$0.973 \pm 0.000$}} & $0.957 \pm 0.000$ \\
\addlinespace[2pt]
Sentiment popularity - AMT & Majority Vote & 1.04 & 20.00 & $0.951 \pm 0.000$ & $0.944 \pm 0.000$ \\
 & Dawid--Skene & 1.04 & 20.00 & $0.943 \pm 0.000$ & $0.944 \pm 0.000$ \\
 & GLAD & 1.04 & 20.00 & $0.947 \pm 0.000$ & \textcolor{darkblue}{\textbf{$0.946 \pm 0.000$}} \\
 & CrowdFM & 1.04 & 20.00 & $0.944 \pm 0.002$ & $0.943 \pm 0.002$ \\
 & $\Phi$-Rasch (OUR) & 1.04 & 20.00 & \textcolor{darkblue}{\textbf{$0.955 \pm 0.000$}} & \textcolor{darkblue}{\textbf{$0.946 \pm 0.000$}} \\
 & PLAT & 1.04 & 20.00 & $0.951 \pm 0.000$ & $0.944 \pm 0.000$ \\
\addlinespace[2pt]
SP & Majority Vote & 1.00 & 5.55 & $0.885 \pm 0.000$ & $0.890 \pm 0.000$ \\
 & Dawid--Skene & 1.00 & 5.55 & $0.920 \pm 0.000$ & $0.915 \pm 0.000$ \\
 & GLAD & 1.00 & 5.55 & $0.913 \pm 0.000$ & \textcolor{darkblue}{\textbf{$0.917 \pm 0.000$}} \\
 & CrowdFM & 1.00 & 5.55 & $0.874 \pm 0.012$ & $0.898 \pm 0.004$ \\
 & $\Phi$-Rasch (OUR) & 1.00 & 5.55 & \textcolor{darkblue}{\textbf{$0.926 \pm 0.000$}} & $0.916 \pm 0.000$ \\
 & PLAT & 1.00 & 5.55 & $0.885 \pm 0.000$ & $0.890 \pm 0.000$ \\
\addlinespace[2pt]
Face & Majority Vote & 1.00 & 8.98 & \textcolor{darkblue}{\textbf{$0.938 \pm 0.000$}} & $0.610 \pm 0.000$ \\
 & Dawid--Skene & 1.00 & 8.98 & $0.904 \pm 0.000$ & $0.626 \pm 0.000$ \\
 & GLAD & 1.00 & 8.98 & $0.925 \pm 0.000$ & $0.611 \pm 0.000$ \\
 & CrowdFM & 1.00 & 8.98 & $0.933 \pm 0.007$ & $0.622 \pm 0.007$ \\
 & $\Phi$-Rasch (OUR) & 1.00 & 8.98 & $0.925 \pm 0.000$ & \textcolor{darkblue}{\textbf{$0.634 \pm 0.000$}} \\
\addlinespace[2pt]
RTE & Majority Vote & 1.00 & 10.00 & $0.910 \pm 0.000$ & $0.919 \pm 0.000$ \\
 & Dawid--Skene & 1.00 & 10.00 & \textcolor{darkblue}{\textbf{$0.948 \pm 0.000$}} & \textcolor{darkblue}{\textbf{$0.927 \pm 0.000$}} \\
 & GLAD & 1.00 & 10.00 & \textcolor{darkblue}{\textbf{$0.948 \pm 0.000$}} & $0.926 \pm 0.000$ \\
 & CrowdFM & 1.00 & 10.00 & $0.880 \pm 0.015$ & $0.914 \pm 0.006$ \\
 & $\Phi$-Rasch (OUR) & 1.00 & 10.00 & $0.943 \pm 0.000$ & $0.926 \pm 0.000$ \\
 & PLAT & 1.00 & 10.00 & $0.910 \pm 0.000$ & $0.919 \pm 0.000$ \\
\addlinespace[2pt]
cifar10h & Majority Vote & 1.00 & 51.42 & \textcolor{darkblue}{\textbf{$0.999 \pm 0.000$}} & $0.992 \pm 0.000$ \\
 & Dawid--Skene & 1.00 & 51.42 & \textcolor{darkblue}{\textbf{$0.999 \pm 0.000$}} & \textcolor{darkblue}{\textbf{$0.993 \pm 0.000$}} \\
 & GLAD & 1.00 & 51.42 & \textcolor{darkblue}{\textbf{$0.999 \pm 0.000$}} & $0.992 \pm 0.000$ \\
 & CrowdFM & 1.00 & 51.42 & $0.998 \pm 0.001$ & $0.992 \pm 0.001$ \\
 & $\Phi$-Rasch (OUR) & 1.00 & 51.42 & $0.994 \pm 0.001$ & $0.989 \pm 0.000$ \\
\addlinespace[2pt]
\end{longtable}


\bibliographystyle{plainnat}
\bibliography{aaai2027}


\end{document}